\documentclass[natbib]{svjour3}                     % onecolumn (standard format)
\smartqed  % flush right qed marks, e.g. at end of proof
\usepackage{graphicx}
\usepackage{caption}
\usepackage{subcaption}
\usepackage{amsmath}
\usepackage{amssymb}
\usepackage{tabularx}
\usepackage{array}
\usepackage{textcomp}
\usepackage{algorithm}
\usepackage{algpseudocode}
\usepackage{url}
\usepackage[english]{babel}
\usepackage{microtype}
\usepackage[utf8]{inputenc}
\usepackage{todonotes}
\newcommand{\argmax}{\operatornamewithlimits{argmax}}
 
\newcommand{\CS}{\mathbf{C}^S}
\newcommand{\cs}{c^S}
\newcommand{\CD}{\mathbf{C}^D}
\newcommand{\cd}{c^D}
\newcommand{\CT}{\mathbf{C}^T}
\newcommand{\ct}{c^T}
\newcommand{\EC}{\boldsymbol{\mu}}
\newcommand{\ec}{\mu}
\newcommand{\secmargin}[1]{\ec_{{#1}..}}
\newcommand{\decmargin}[1]{\ec_{.{#1}.}}
\newcommand{\tecmargin}[1]{\ec_{..{#1}}}
\newcommand{\sdeg}{\delta^S}
\newcommand{\ddeg}{\delta^D}

\newcommand{\SDEG}{\boldsymbol{\delta}^S}
\newcommand{\DDEG}{\boldsymbol{\delta}^D}
\newcommand{\Like}{\mathcal{L}}
\newcommand{\mmodel}{\nu}
\newcommand{\mdata}{m}
\newcommand{\modldiss}{\Delta_{\textnormal{MODL}}}  
\newcommand{\sprob}{\mathbb{P}^S}
\newcommand{\scprob}{\mathbb{P}^S_C}
\newcommand{\dcprob}{\mathbb{P}^D_C}
\newcommand{\tcprob}{\mathbb{P}^T_C}
\newcommand{\sdcprob}{\mathbb{P}^{S,D}_C}
\newcommand{\fullcprob}{\mathbb{P}^{S,D,T}_C}

\begin{document}

\journalname{Advances in Data Analysis and Classification}

\title{Discovering Patterns in Time-Varying Graphs:\\ A Triclustering Approach}

\titlerunning{Discovering Patterns in Time Varying Graphs}        % if too long for running head

\author{Romain Guigour\`es \and Marc Boull\'e \and \\ Fabrice Rossi}

\authorrunning{R. Guigour\`es \and M. Boull\'e \and F. Rossi} % if too long for running head

\institute{R. Guigour\`es and M. Boull\'e \at
							Orange Labs\\
              2 avenue Pierre Marzin \\
              22300 Lannion, France\\
              \email{romain.guigoures@gmail.com}\\
              \email{marc.boulle@orange.com}         %  \\
%             \emph{Present address:} of F. Author  %  if needed
           \and
F. Rossi \at
SAMM EA 45 43\\
Universit\'e Paris 1\\
90 rue Tolbiac\\
75013 Paris, France\\
\email fabrice.rossi@univ-paris1.fr}

\date{}
% The correct dates will be entered by the editor

\maketitle

\begin{abstract}
  This paper introduces a novel technique to track structures in time varying
  graphs. The method uses a maximum a posteriori approach for adjusting a
  three-dimensional co-clustering of the source vertices, the destination
  vertices and the time, to the data under study, in a way that does not require any 
  hyper-parameter tuning. 
  The three dimensions are simultaneously segmented in order to build clusters of
  source vertices, destination vertices and time segments where the edge
  distributions across clusters of vertices follow the same evolution over the time
  segments. The main novelty of this approach lies in that the time segments
  are directly inferred from the evolution of the edge distribution between
  the vertices, thus not requiring the user to make any a priori
  quantization. Experiments conducted on artificial data illustrate the good
  behavior of the technique, and a study of a real-life data set shows the
  potential of the proposed approach for exploratory data analysis.
\end{abstract}


\newpage
\section{Introduction}
In real world problems, interactions between entities are generally evolving
through time. This is the case for instance in transportation networks (roads,
train, etc) or communication networks (mobile phone, web, etc). Understanding
the corresponding time evolving interaction graphs implies both to discover
structures in those graphs and to track the evolution of those structures
through time.
% reviewer: 2
In a subway network for example, entities are the stations and interactions are 
the passenger journeys from an origin to a destination station at a given start time.
Understanding the evolving distribution of journeys over time is of great help
for network planners, for instance to schedule trains efficiently. 

Early works on the structure of the interactions in graphs dates back to the
$1950$s in the context of social networks analysis: \cite{Nadel1957} proposes
to group the actors that play similar \textit{roles} within the network. The
clustering of vertices -- that models the actors -- has been extensively
studied. The vast literature on graph partitioning is surveyed in 
such as the one of \cite{Schaeffer2007}, \cite{Goldenberg2009} and
\cite{Fortunato2010}, among others. 

The analysis of time-varying/time-evolving/dynamic graphs is
quite recent \citep{CasteigtsEtAl2012}. \cite{Hopcroft2004} have been first
interested in the evolution of the vertices clustering. In their approach, a
time-varying graph is modeled by a sequence of static graphs in which the
clusters are retrieved using an agglomerative hierarchical clustering, where
the similarity between the clusters is a cosine \citep{Jain1998}. Then, the
evolution of the clusters across the snapshots is investigated.  In more
recent works, \cite{Palla2007} adapt their own Clique Percolation Method
\citep{Palla2005} to time-evolving graphs by exploiting the overlap of the
clusters at $t$ and $t+1$ to study their evolution through time.
\cite{Xing2010} use a probabilistic approach to study the evolution of the
membership of each vertex to the clusters.  As for \cite{Faloutsos2007}, they
have introduced an information-theoretic based approach named
\textit{Graphscope}. It is a two-stage method dedicated to simple bipartite
graphs that tracks structures within time-varying graphs. First, a partition
of the snapshots is retrieved and evaluated using a MDL framework
\citep{Grunwald2007}, then an agglomerative process is used to determine the
temporal segmentation.  As discussed by \cite{Lang2009}, the partitioning 
results may be sensitive to the coding schemes: in particular, coding schemes
like those used by \cite{Faloutsos2007} have no guarantee of robustness
w.r.t. random graphs.  

The approaches introduced above focus on a specific way of introducing time
evolution into interaction analysis: they study a sequence of static
interaction graphs. This is generally done via a quantization of the time
which turns temporal interaction with possibly continuous time stamps into
said sequence of graphs. The quantization is mainly \emph{ad hoc}, generally
based on ``expert'' or ``natural'' discrete time scales (such as hourly graphs
or daily graphs) which lead to snapshots of the temporal interaction
structure. Then the clusters of vertices are detected separately from the time
quantization step hiding potential dependencies between those two aspects,
as well as possible intricate temporal patterns. \cite{Fortunato2010} has
raised these problems and considers more suitable the approaches that track
the clusters of vertices and the temporal structure in one unique step.

\textit{Co-clustering} is a way to address this requirement. This technique aims at simultaneously partitioning the variables
describing the occurrences in a data set \citep{Hartigan1972}. 
Co-clustering has been 
applied to gene expressions problems \citep{VanMechelen2004}
and has been widely used in documents classification \citep{Dhillon2003},
among other applications. An example of
the application of co-clustering to graphs is given by \cite{Rege2006} in the
case of static graphs. In this type of approaches, the graph is represented by
its \textit{adjacency matrix}: the rows and the columns correspond to the
vertices and the values in the cells quantify the edge intensities between two
vertices. The simultaneous partitioning of rows and columns coincide with the
clusters of vertices. One advantage of co-clustering is that it is able to deal
with nominal and numerical variables \citep{Bekkerman2005,Nadif2010}. Thus,
co-clustering approaches for static graphs can be adapted to time-evolving
graphs by introducing a third variable with temporal information. Such an
approach was explored in \cite{Zaki2005} in order to study the temporal
evolution of micro-array data. While the algorithm defined in this paper,
\emph{TriCluster}, uses the three-mode representation idea it aims at finding
patterns rather than at clustering the three dimensions together. Therefore
it shares only its data representation paradigm with the approach presented in
the present paper. A closer technique is presented by \cite{Schepers2006}
who  introduce a three-mode partitioning approach. They define a three
dimensional block model, that is optimized by minimizing a least squares loss function. To that end, the performances of several algorithms are investigated. The results shows that partitioning simultaneously all the three dimension provides better results than dealing with them independently. Moreover, \cite{Schepers2006} point out the difficulty of optimizing their global criterion and discuss the benefits of a multistart procedure. This is also treated in the present paper.

In this paper, we propose an approach for time-varying graphs built upon the
MODL approach of \cite{Boulle2010}. Our method groups vertices based on
similarity between connectivity patterns at the cluster level. In addition, it
partitions the time interval into time segments during which connectivity
patterns between the clusters are stationary. This corresponds to a
triclustering structure which is optimized jointly in our method, without
introducing any user chosen hyper-parameter (in particular, the number of
clusters is chosen automatically). This approach is resilient to noise and
reliable in the sense that no co-clustering structure is detected in case of
uniform random graphs (e.g. \cite{erdos_renyi_ORGI1959}) and that no 
time segmentation is retrieved in case of stationary graphs. In addition, the
true underlying distribution is asymptotically estimated. 

The rest of the paper is organized as follows. Section
\ref{sec:temp-inter-data} introduces the type of temporal interaction data our
model can handle. A combinatorial generative model for such data is described
in Section \ref{sec:gener-model-temp}. Section \ref{sec:parameter-estimation}
presents our Maximum A Posteriori strategy for estimating the parameters of
this model from a temporal data set. Section \ref{sec:exper-artif-data}
investigates the behavior of the method using artificial data. Finally, the
method is applied on a real-life data set in order to show its effectiveness
on a practical case in Section \ref{sec:exper-real-life}.
Finally, Section~\ref{sec:conclusion} gives a summary and suggests future work. 

\section{Temporal Interaction Data and Time-Varying Graph}\label{sec:temp-inter-data}
In this paper, we study interactions between entities that take place during a 
certain period of time. We assume given two finite sets $S$ and $D$ which are
respectively the set of sources (entities from which interactions start) and
the set of destinations (entities to which interactions are destined). Each
interaction is a triple $(s,d,t)\in S\times D\times \mathbb{R}$ where $t$ is
the instant at which the interaction takes place (in general $t$ is called the
time stamp of the interaction). In this paper a temporal interaction data set
is a finite set $E\subset S\times D\times \mathbb{R}$ made of $\mdata$ interaction
triples, $(s_n,d_n,t_n)_{1\leq n\leq \mdata}$.

Time stamps are assumed to be measured with enough precision to ensure that
each of the $t_n$ is unique among the $(t_j)_{1\leq j\leq \mdata}$ and thus
the third variable of a temporal interaction data set could be seen as a
continuous variable. However, to avoid contrast related effects and to
simplify data modeling, we use a rank based transformation: each $t_n$ is
replaced by its rank in $(t_j)_{1\leq j\leq \mdata}$, leading to an integer
valued variable. 

As pointed about in the introduction, interaction data are frequently
represented in graph forms. Taking into account the temporal aspect of
interactions has led to the introduction of several notions of time-varying 
(or dynamic, or evolving) graphs. A unifying framework is proposed in
\cite{CasteigtsEtAl2012} and can be specialized to address different temporal
notions. In this framework, a temporal interaction data set $E$ as defined
above corresponds to a time-varying graph given by the triple $\mathcal{G} =
(V,F,\rho)$, where $V=S\cup D$ is the set of vertices of the graph,
$F=\{(s,d)\in S\times D | \exists t\in\mathbb{R}, (s,d,t)\in E\}$ is the
projection of $E$ on $S\times D$ (giving the edges of the graph) and where the
presence function $\rho$ from $F\times \mathbb{R}$ to $\{0,1\}$ is given by
\begin{equation}
\rho(s,d,t)=\left\{
  \begin{array}{cl}
1& \text{if }(s,d,t)\in E,\\
0& \text{if }(s,d,t)\not\in E.  
  \end{array}
\right. 
\end{equation}
Thus, $E=(s_n,d_n,t_n)_{1\leq n\leq \mdata}$ can be seen a particular case of
time-varying graph, a fact that will prove useful in order to define a
generative model for such temporal interaction data. In this context the pair
of terms ``entity'' and ``vertex'', as well as the pair of terms ``edge'' and
``interaction'', are interchangeable. Nevertheless, we will standardize on the
graph related terminology (vertex and edge) to avoid confusion.

Notice that the temporal interaction data notion used here is quite general as
it can lead to simple directed graphs (where $S=D$ in general), but also to
bipartite graphs (when $S\cap D=\emptyset$). In addition, temporal interaction
data and thus time-varying graphs are inherently multigraphs (using the graph 
theory term): provided they have different time stamps, two edges can have
exactly the same source and destination vertices, allowing this way multiple
interactions to take place between the same actors at different moments. In
addition, undirected graphs can also be studied under this general paradigm.

Notice also that while we use interchangeably the terms ``temporal graph'',
``time-evolving graph'' and ''time-varying graph'', the first one is more
accurate than the others in the sense that we are studying a (multi)graph with
temporal information rather than e.g. a time series of graphs. Indeed each
time stamp is attached to one interaction rather than to a full
graph. However, we use also the terms ``time-evolving graph'' because we look
for time intervals in which the interaction pattern is stationary leading to a
time series of such fixed interaction patterns which can be seen as a
time-evolving graph (but at a coarser grain). By interaction pattern we mean
here a high level structure in a static graph, as seen in e.g. stochastic
block models \citep{Snijders2001}: for instance, in some situations, one might
partition the vertices into clusters such that the graph contains a small
number of edges between members of different clusters and a high number of
edges between members of the same cluster (this is a modular structure as
looked for by community detection algorithms see
e.g. \cite{Fortunato2010}). Figure 1 gives an example of four such patterns.

\section{A Generative Model for Temporal Interaction
  Data}\label{sec:gener-model-temp}
We propose in this paper a probabilistic modeling  \citep{murphy2012machine} of temporal interaction
data: we introduce a probabilistic model
that can generate data that resemble the observed data. The present Section
describes the model in details while Section \ref{sec:parameter-estimation}
explains how to fit the model to a given data by estimating its parameters. 

The model is inspired by the graph view of the data. As in a static
graph data analysis, we aim at producing a form of block model in which source
entities/vertices and destination entities/vertices are partitioned into
homogeneous classes (in terms of connectivity patterns). Therefore, the model
is based on a partition of the source set $S$ and on a partition of the
destination set $D$. Time is handled via a piecewise stationary
assumption. The model uses a partition of the time stamp ranks,
$\{1,\ldots,\mdata\}$, into consecutive subsequences (which correspond to time
intervals). Each subsequence is associated to a specific block model.

The initial view of the data as a three dimensional data set allows one to
interpret the block models as a triclustering. Indeed, each source vertex,
each destination vertex and each time stamp belongs to a cluster of the
corresponding set (respectively $S$, $D$ and $\mathbb{R}$). In addition,
clusters of time stamps respect the natural ordering of time (as they are
consecutive subsequences).

As described below, the model is based on a combinatorial view of temporal
interaction data rather than on the continuous parameter based model used in
classical block models. It is based on the MODL approach of \cite{Boulle2010}
which addresses density estimation via this type of combinatorial model.

\subsection{Notations and definitions}
In order to define our generative model, we need first to introduce some 
notations and vocabulary. Given a set $A$, $|A|$ is the cardinality of $A$. As 
explained in Section \ref{sec:temp-inter-data}, time stamps are transformed
into ranks. Thus the set of time stamps is $\{1, 2, \ldots, \mmodel\}$ where $\mmodel$
is the number of edges/interactions\footnote{To avoid confusion, we denote $\mmodel$
the number of edges as a parameter of the model and $\mdata$ the
number of edges in a given data set.}. A partition of $\{1, 2, \ldots, \mmodel\}$ respects
its ordering if and only if given any pair of distinct classes of the
partition, $c_1$ and $c_2$, all the elements of $c_i$ are smaller than all the
elements of $c_j$ either for $i=1$ and $j=2$ or for $i=2$ and $j=1$. Obviously,
classes of a partition that respects the order of $\{1, 2, \ldots, \mmodel\}$ are
consecutive subsequences of $\{1, 2, \ldots, \mmodel\}$. We call any such
consecutive subsequence an \emph{interval} because it represents a time
interval in the original data set. For instance the subsequence $\{1, 2, 3\}$
represents the time interval ranging from the oldest time stamp in the data
set (the first one) to the third one in the data set.

Given three sets $A$, $B$ and $C$ and three partitions $P_A$,
$P_B$ and $P_C$ of those sets, a \emph{tricluster} is the Cartesian product of a class of
each partition, that is $a\times b\times c$ with $a\in P_A$, $b\in P_B$ and
$c\in P_c$. It is a subset of $A\times B\times C$ by construction, and the set
of all triclusters generated by $P_A$, $P_B$ and $P_C$ forms a partition of
$A\times B\times C$, called a \emph{triclustering}. 

For instance if $A=\{x, y, z\}$, $B=\{1, 2, 3, 4\}$ and $C=\{\alpha, \beta\}$,
elements of $A\times B\times C$ are the triplets $(x, 1, \alpha)$,
$(z, 3, \beta)$, etc. A way to build a very structured clustering, called a
triclustering, of $A\times B\times C$ consists in building three clusterings:
one for $A$, e.g. $A=\{x, z\}\cup \{y\}$, one for $B$, e.g.
$B=\{1, 2\}\cup\{3, 4\}$ and one for $C$, e.g. $C=\{\alpha\}\cup
\{\beta\}$.
Then the clustering of $A\times B\times C$ if made of the Cartesian products
of the clusters of $A$, $B$ and $C$. One of such cluster is
$\{x, z\}\times \{1, 2\}\times \{\alpha\}$ which contains the following
triplet:
\[
\{(x, 1, \alpha), (x, 2, \alpha), (z, 1, \alpha), (z, 2, \alpha)\}. 
\]
Other clusters of this clustering are $\{x, z\}\times \{3, 4\}\times
\{\beta\}$, etc. 

\subsection{Model parameters}
As explained above, our generative model is based on a triclustering. The
partitions of the source and destination vertices are considered as parameters
of the model, together with a series of other parameters described below. We
list here all the parameters, but consistency constraints on the model prevent
those parameters to be chosen arbitrarily. The constraints and our choice of
free parameters are explained in the next subsection.

In the end, all parameters will have been estimated on the basis of the data.

Given a set of source vertices $S$, a set of destination vertices $D$,
the model uses the following parameters:
\begin{enumerate}
\item $\mmodel$, the number of edges to generate;
\item $\CS=(\cs_1,\ldots,\cs_{k_S})$, the partition of the source
  vertices into $k_S$ clusters;
\item $\CD=(\cd_1,\ldots,\cd_{k_D})$, the partition of the destination
  vertices into $k_D$ clusters;
\item $\CT=(\ct_1,\ldots,\ct_{k_T})$, the partition of the time stamp ranks
  $\{1,\ldots,\mmodel\}$ into $k_T$ clusters. This partition must respect the order
  of the ranks (clusters are intervals/consecutive subsequences);
\item $\EC=\{\ec_{ijl}\}_{1 \leq i \leq k_S, 1 \leq j \leq k_D, 1 \leq l \leq
    k_T}$, the number of edges that will be generated
  by the tricluster indexed by $(i,j,l)$. More precisely, for each tricluster
  $\cs_i\times\cd_j\times \ct_l$ the model will generate $\ec_{ijl}$ edges with sources
  in $\cs_i$, destinations in $\cd_j$ and time stamps in $\ct_l$;
\item $\SDEG=\{\sdeg_{s}\}_{s \in S}$, the out-degree of each source vertex
  $s$. In other words, $\sdeg_s$ is the number of edges generated by
  the model for which the source vertex is $s$;
\item $\DDEG=\{\ddeg_{d}\}_{d\in D}$, the in-degree of each destination
  vertex. In other words, $\ddeg_d$ is the number of edges generated by
  the model for which the destination vertex is $d$. 
\end{enumerate}
Notice that $\CS$, $\CD$ and $\CT$ build a triclustering of the set $S\times
D\times \{1,\ldots,\mmodel\}$. Each tricluster consists here in a cluster of
source vertices, a cluster of destination vertices and an interval of time
stamp ranks. 

\subsection{Constrained and free parameters}
The parameters described in the previous subsection have to satisfy some
constraints. The most obvious one links $\EC$ to $\mmodel$ by
\begin{equation}
\mmodel=\sum_{1 \leq i \leq k_S, 1 \leq j \leq k_D, 1 \leq l \leq k_T}\ec_{ijl}.
\end{equation}
To introduce the other constraints, we will use classical marginal count 
notations applied to the three dimensional array $\EC$, that is
\begin{align}
\secmargin{i}&=\sum_{1 \leq j \leq k_D, 1 \leq l \leq k_T}\ec_{ijl},\\
\decmargin{j}&=\sum_{1 \leq i \leq k_S, 1 \leq l \leq k_T}\ec_{ijl},\\
\tecmargin{l}&=\sum_{1 \leq i \leq k_S, 1 \leq j \leq k_D}\ec_{ijl}.
\end{align}
In theses notations, a dot $.$ indicates that a sum is made over all possible
values of the corresponding index.

Degrees must be consistent with edges produced by each cluster. We
have therefore
\begin{equation}
\forall i \in\{1,\ldots,k_S\}, \sum_{s\in\cs_i}\sdeg_{s}=\secmargin{i},
\end{equation}
and
\begin{equation}
\forall j \in\{1,\ldots,k_D\}, \sum_{d\in\cd_j}\ddeg_{d}=\decmargin{j}.
\end{equation}
Indeed, all the edges that have a source in e.g. $\cs_i$ must have been
generated by triclusters of the form $\cs_i\times\cd\times\ct$ where $\cd$ and
$\ct$ are arbitrary clusters of destination vertices and time stamps,
respectively. The left hand part of the equation counts those edges by
summing the degrees in $\cs_i$ while the right hand part counts them by
summing the edge counts in the triclusters.

There is a much stronger link between $\CT$ and $\EC$. As for the other
clusters, marginal consistency is needed and therefore we have
\begin{equation}
\forall l\in\{1,\ldots,k_T\}, \left|\ct_l\right|=\tecmargin{l}.
\end{equation}
The consistency equation is simpler than in the case of source/destination
clusters because the time stamp ranks are unique and there is no ``degree''
attached to them. 

In addition, as $\CT$ respects the order of $\{1,\ldots,\mmodel\}$, its
classes can be reordered such that $\ct_1$ contains the smallest ranks,
$\ct_2$ the second smallest ranks, etc. Then as the classes are consecutive
subsequences, the only possible partition is given by
\begin{equation}
\CT=\left(\left\{1,\ldots,\tecmargin{1}\right\},\left\{\tecmargin{1}+1,\ldots,\tecmargin{1}+\tecmargin{2}\right\},\ldots,\left\{\sum_{l=1}^{k_t-1}\tecmargin{l}+1,\ldots,\mmodel\right\}\right).
\end{equation}

In practical terms, this means that up to a renumbering of its classes, there
is a unique partition $\CT$ of the time stamp ranks that respects their order
and that is compatible with a given $\EC$. Then $\CT$ can be seen as a bound
parameter. Notice that we could on the contrary leave $\CT$ free and then
obtain constraints on $\EC$. This would be more complex to handle in terms of
the prior distribution on the parameters. 

In the rest of the paper, we denote $\mathcal{M}$ a complete 
list of values for the free parameters of the model, that is
$\mathcal{M}=(\mmodel,\CS,\CD,\EC,\SDEG,\DDEG)$. We assume implicitly that
$\mathcal{M}$ fulfills the constraints outlined above. In addition, even when
we use this choice of free parameters, a value of $\mathcal{M}$ will be called
a \emph{triclustering}. In particular, $\CT$ will always denote the time
stamp partition uniquely defined by $\mathcal{M}$. We will also always denote
$k_S$, $k_D$ and $k_T$ the number of clusters in each of the three
partitions.

\paragraph{An example:}
to illustrate the parameter space, a simple example is described below. The
source set is $S=\{1,\ldots,6\}$ and the destination set is $D=\{a, b, \ldots,
h\}$. We fix $\mmodel=50$ and thus the time stamp ranks form the set
$\{1,\ldots, 50\}$. We choose 3 source clusters 
\[
\CS=\{\{1, 2, 3\}, \{4, 5\}, \{6\}\},
\]
2 destination clusters 
\[
\CD=\{\{a, b, c, d, e\},\{f, g, h\}\},
\]
and 3 time clusters (unspecified yet as they will be consequences of $\EC$). A
possible choice for $\EC$ is given by the following tables
\begin{center}
  \begin{tabular}{c|c|c|}
    &$\cd_1$&$\cd_2$\\\hline
    $\cs_1$ &5 & 1\\\hline
    $\cs_2$ &2 & 0\\\hline
    $\cs_3$ &4 & 0 \\\hline
    \multicolumn{3}{c}{$\ct_1$}
  \end{tabular}
  \hspace{1em}
  \begin{tabular}{c|c|c|}
    &$\cd_1$&$\cd_2$\\\hline
    $\cs_1$ &2 & 2\\\hline
    $\cs_2$ &2 & 5 \\\hline
    $\cs_3$ &5 & 5\\\hline
    \multicolumn{3}{c}{$\ct_2$}
  \end{tabular}
  \hspace{1em}
  \begin{tabular}{c|c|c|}
    &$\cd_1$&$\cd_2$\\\hline
    $\cs_1$ &0 &0\\\hline
    $\cs_2$ &1 &0\\\hline
    $\cs_3$ &1 &15\\\hline
    \multicolumn{3}{c}{$\ct_3$}
  \end{tabular}
\end{center}
There is one table per time stamp interval and in each table the rows
correspond to the three source clusters while the columns correspond to the
two destination clusters. For instance $\ec_{111}=5$. Notice that the sum of
all the numbers in the table cells equals $\mmodel=50$, as imposed by the constraints. 

Marginal counts induced by $\EC$ are then
\[
\begin{array}{l|ccc}
i&1&2&3\\\hline
\secmargin{i}&10&10&30
\end{array}
\quad\quad
\begin{array}{l|cc}
j&1&2\\\hline
\decmargin{j}&22&28
\end{array}
\quad\quad
\begin{array}{l|ccc}
l&1&2&3\\\hline  
\tecmargin{l}&12&21&17
\end{array}
\]

They are compatible, for instance, with the following out degrees $\SDEG$
\[
  \begin{array}{l|*{3}{c}|*{2}{c}|c}
s&    1 & 2 & 3 & 4 & 5 & 6  \\\hline
\sdeg_s &3 & 6 & 1 & 2 & 8 & 30 
  \end{array}
\]
and in degrees $\DDEG$
\[
  \begin{array}{l|*{5}{c}|*{3}{c}}
d&    a & b & c & d & e & f & g & h \\\hline
\ddeg_d &3 & 6 & 2 & 6 & 5 & 13 & 8 & 7 
\end{array}
\]
As explained above, the only possible time stamp rank partition is then 
\[
\CT=\{\{1,\ldots,12\},\{13,\ldots,33\},\{34,\ldots,50\}\}.
\]

\subsection{Data generating mechanism in the proposed model} 

Given the parameters $\mathcal{M}$, a temporal data set
$E=(s_n,d_n,t_n)_{1\leq n\leq \mmodel}$ is generated by a hierarchical distribution
build upon uniform distributions. 
\begin{enumerate}
\item The $\mmodel$ edges are generated by first choosing which one of
  $k_S \times 
  k_D \times k_T$ triclusters is responsible for generating each of the
  edges. This is done by assigning 
  each of the $\mmodel$ edges to a tricluster under the constraints given by the
  assignment $\EC$. All compatible mappings from edges to triclusters are
  considered equiprobable. Then a given mapping $E_{map}$ has a probability of one
  divided by the number of compatible mappings, that is:
  \begin{equation}
P(E_{map}|\mathcal{M})=\dfrac{ \prod_{i=1}^{k_S} \prod_{j=1}^{k_D} \prod_{l=1}^{k_T} \ec_{ijl}!}{\mmodel!}.
\end{equation}

\item In two independent second steps, edges are mapped to source vertices and
  destination vertices. Indeed, each source cluster $C^S_i$ is responsible for
  generating $\secmargin{i}$ edges under the assignment constraints specified by
  the degrees of the source vertices (and similarly for destination
  vertices). As in the previous step, all mappings from the edges assigned to
  a cluster to its vertices that are compatible with the assignment are
  considered equiprobable. In addition, mappings are independent from cluster
  to cluster. Then a given source  mapping $S_{map}$ and a destination mapping
  $D_{map}$ have the following probabilities:
  \begin{equation}
P(S_{map}|\mathcal{M})=\dfrac{  \prod_{s\in S}\sdeg_s!}{
  \prod_{i=1}^{k_S}\secmargin{i}!},\quad P(D_{map}|\mathcal{M})=\dfrac{ \prod_{d \in D}\ddeg_d!}{\prod_{j=1}^{k_D}\decmargin{j}!}.
\end{equation}

\item Based on the previous steps, each edge has now a source vertex and a
  destination vertex. Its time stamp is obtained in a similar but simpler marginal
  procedure. Indeed inside a time interval, we simply order the edges in an arbitrary
  way, using a uniform probability on all possible orders. Orders are also
  independent from one interval to another. Then a given time ordering of the
  edges $T_{order}$ has a probability:
  \begin{equation}
P(T_{order}|\mathcal{M})=\dfrac{1}{ \prod_{l=1}^{k_T} \tecmargin{l}!}.
\end{equation}
\end{enumerate}
\paragraph{An example (continued):} using the parameter list given as an example in the previous subsection, we
can generate a temporal data set. As a first step, we assign the 50 edges to
the 18 triclusters (in fact only to the 13 triclusters with non zero values in
$\EC$). To simplify the example, we choose the assignment in which edges are
generated from 1 to 50 by the first available tricluster in the lexicographic
order on the indexing triple $(i,j,l)$. This means that edges 1 to 5  are generated by
the tricluster $(1,1,1)$, that is $\cs_1\times\cd_1\times\ct_1$, then edges 6 and
7 are generated by tricluster $\cs_1\times\cd_1\times\ct_2$, then edge 8 by
tricluster $\cs_1\times\cd_2\times\ct_1$ (we skip
$\cs_1\times\cd_1\times\ct_3$ because $\ec_{113}=0$), etc. This is summarized
in the following tables:
\begin{center}
  \begin{tabular}{c|c|c|}
    &$\cd_1$&$\cd_2$\\\hline
    $\cs_1$ &$\{1, \ldots, 5\}$ & $\{8\}$\\\hline
    $\cs_2$ &$\{11, 12\}$ & $\emptyset$\\\hline
    $\cs_3$ &$\{21, \ldots, 24\}$ & $\emptyset$ \\\hline
    \multicolumn{3}{c}{$\ct_1$}
  \end{tabular}
  \hspace{1em}
  \begin{tabular}{c|c|c|}
    &$\cd_1$&$\cd_2$\\\hline
    $\cs_1$ &$\{6, 7\}$ & $\{9, 10\}$\\\hline
    $\cs_2$ &$\{13, 14\}$ & $\{16, \ldots, 20\}$ \\\hline
    $\cs_3$ &$\{25, \ldots, 29\}$ & $\{31, \ldots, 35\}$\\\hline
    \multicolumn{3}{c}{$\ct_2$}
  \end{tabular}
  \hspace{1em}
  \begin{tabular}{c|c|c|}
    &$\cd_1$&$\cd_2$\\\hline
    $\cs_1$ &$\emptyset$ &$\emptyset$\\\hline
    $\cs_2$ &$\{15\}$ &$\emptyset$\\\hline
    $\cs_3$ &$\{30\}$ &$\{36, \ldots, 50\}$\\\hline
    \multicolumn{3}{c}{$\ct_3$}
  \end{tabular}
\end{center}
The edges are assigned variable per variable. For instance vertices in
$\cs_1$ are the source vertex for the following edges
\[
\{1, \ldots, 5\} \cup \{8\} \cup \{6, 7\} \cup \{9, 10\}=\{1, \ldots, 10 \}.
\]
Using the degree constraints $\sdeg_1, \sdeg_2$ and $\sdeg_3$, one possible
assignment is
\[
  \begin{array}{r|*{10}{c|}}
\text{edge} & 1 & 2 & 3 & 4 & 5 & 6 & 7 & 8 & 9 & 10\\\hline
\text{source}&    2 & 2 & 1 & 2 & 1 & 3 & 2 & 1 & 2 & 2
  \end{array}
\]
Similarly, vertices in $\cd_1$ are the destination vertex for the following
edges
\[
\{1, \ldots, 5\} \cup \{6, 7\} \cup \{11, 12\} \cup \{13, 14\} \cup \{15\} \cup 
\{21, \ldots, 24\} \cup \{25, \ldots, 29\}  \cup \{30\},
\]
which can be obtained using the following assignment
\[
  \begin{array}{r|*{22}{c|}}
\text{edge} & 1 & 2 & 3 & 4 & 5 & 6 & 7 & 11 & 12 & 13 & 14 & 15 & 21 & 22 & 23 & 24 & 25 & 26 & 27 & 28 & 29 & 30
\\\hline
\text{destination}&  d & d & e & a & b & a & b & e & d & d & b & b & b & d & a & e & c & d & e & e & b & c 
  \end{array}
\]
Finally, time stamp ranks are assigned in a similar way. For instance time
stamp ranks from $\{1, \ldots, 12 \}$ are assigned to edges
\[
\{1, \ldots, 5\} \cup \{8\} \cup \{11, 12\} \cup \{21, \ldots, 24\},
\]
for instance by
\[
  \begin{array}{r|*{12}{c|}}
\text{edge} & 1 & 2 & 3 & 4 & 5 & 8 & 11 & 12 & 21 & 22&23&24\\\hline
\text{time stamp rank}&   5 & 7 & 10 & 4 & 8 & 2 & 9 & 6 & 1 & 3 & 12 & 11 
  \end{array}
\]
At the end of this process, a full temporal data set is generated. In our
working example, the first five edges are
\[
\begin{array}{c|ccc}
\text{edge}&\text{source}&\text{destination}&\text{time stamp rank}\\\hline
1&2&d&5\\ 
2&2&d&7\\
3&1&e&10\\
4&2&a&4\\
5&1&b&8
\end{array}
\]

\subsection{Likelihood function}
Because of the combinatorial nature of the proposed model, the likelihood
function has a peculiar form. Let $E=(s_n,d_n,t_n)_{1\leq n\leq \mdata}$  be a
temporal data set. The likelihood function $\Like(\mathcal{M}|E)$ takes a non
zero value if and only if $\mathcal{M}$ and $E$ are compatible according to
the following definition. 
\begin{definition}\label{def:compatibility}
A temporal data set $E=(s_n,d_n,t_n)_{1\leq n\leq \mdata}$ and a parameter
list $\mathcal{M}=(\mmodel,\CS,\CD,\EC,\SDEG,\DDEG)$ are \emph{compatible} if and
only if:
\begin{enumerate}
\item $\mdata=\mmodel$;
\item for all $s\in S$, $\sdeg_s=|\{n\in\{1,\ldots,\mdata\}|s_n=s\}|$;
\item for all $d\in D$, $\ddeg_d=|\{n\in\{1,\ldots,\mdata\}|d_n=d\}|$;
\item for all $i\in\{1,\ldots,k_S\}$, $j\in\{1,\ldots,k_D\}$ and
  $l\in\{1,\ldots,k_T\}$,
\begin{equation}
\EC_{ijl}=\left|\left\{\{n\in\{1,\ldots,\mdata\}|s_n\in\cs_i, d_n\in\cd_j, t_n\in\ct_l\right\}\right|.
\end{equation}
\end{enumerate}
\end{definition}
Based on this definition, the likelihood function is equal to zero when
$\mathcal{M}$ and $E$ are not compatible and is given by the following formula
when they are compatible
\begin{equation}
\Like(\mathcal{M}|E)=\dfrac{ \left(\prod_{i=1}^{k_S} \prod_{j=1}^{k_D}
  \prod_{l=1}^{k_T} \ec_{ijl}!\right)\left(\prod_{s\in
    S}\sdeg_s!\right)\left(\prod_{d \in D}\ddeg_d!\right)}{\mmodel!
\left(\prod_{i=1}^{k_S}\secmargin{i}!\right)
\left(\prod_{j=1}^{k_D}\decmargin{j}!\right)\left(\prod_{l=1}^{k_T}\tecmargin{l}!\right)}.
\end{equation}
Notice that while the formula is expressed in terms of the parameters
$\mathcal{M}$ only, it depends obviously on the characteristics of the data
set $E$, via the compatibility constraints between $\mathcal{M}$ and $E$.

One of the interesting properties of the likelihood function is that it
increases when the block structure associated to the triclustering
``sharpens'' in the following sense: the likelihood increases when the number
of empty triclusters ($\ec_{ijl}=0$) increases.  

\section{Parameter estimation}\label{sec:parameter-estimation}
In order to adjust the parameters $\mathcal{M}$ of our model to a temporal
data set $E$, we use a Maximum A Posteriori (MAP) approach where the estimator
for the parameters is given by
$\mathcal{M}^*=\argmax_{\mathcal{M}}P(\mathcal{M}) P(E|\mathcal
{M})$. Together with a non informative prior distribution on the parameters,
this enables us to adjust all the parameters of the model without introducing
any user chosen hyper-parameter. In addition, the chosen prior distribution
penalizes complex models to limit the risk of overfitting.

The model is designed in such a way that when the number of edges in
$\mathcal{M}$ is $\mmodel$, then all temporal data sets generated have exactly
$\mmodel$ edges. In an estimation context, we fix therefore
directly $\mmodel=\mdata$ where $\mdata$ is the observed number of
edges. This can be seen as fixing $\mmodel$ to its MAP estimate as the
likelihood of $\mathcal{M}$ given $E$ is zero when $\mmodel\neq\mdata$. For
the rest of the parameters, we specify a non informative prior distribution as
follows. 

\subsection{Prior distribution on the parameters}
The prior is built hierarchically and uniformly at each stage in order to be
uninformative. This is done as follows: 
\begin{enumerate}
\item For source and destination partitions, a maximal number of clusters is drawn uniformly at
  random between $1$ and the cardinality of the set to cluster (for instance 
  $|S|$ for the set of source vertices). For the time stamps partition, the number
  of clusters is drawn in the same way. We obtain this way $k_S^{\max}$, $k_D^{\max}$
  and $k_T$ with the associated probability distribution:
  \begin{equation}
p(k_S^{\max}) = \dfrac{1}{|S|}, \quad p(k_D^{\max}) = \dfrac{1}{|D|} , \quad p(k_T) = \dfrac{1}{\mdata}.    
  \end{equation}
  The case with one single cluster corresponds to the \emph{null
    triclustering}, where there is no significant pattern within the
  graph. The other extreme case corresponds to the most refined triclustering
  where each vertex plays a role that is significantly specific to be clustered alone: the
  triclustering has as many clusters as vertices (on both source and
  destination). In social networks analysis, both
  extreme clustering structures are consistent with the notion of
  \textit{regular equivalence} 
  introduced in the works of \cite{White1983} and \cite{Borgatti1988}. 

  The case with one time segment corresponds to a stationary graph over
  time. The one with as many time segments as edges is an extremely
  fine-grained quantization: as time is a continuous variable, this case is
  allowed in our approach. It can appear when the connectivity patterns are
  gradually changing over time in a very smooth way, see Section
  \ref{sec:exper-artif-data} for an example.

  Notice that this prior is given for the sake of mathematical soundness, but
  in practice, it has no effect on the MAP criterion as it does not depend on
  the actual values $k_S^{\max}$, $k_D^{\max}$ and $k_T$, but only on fixed
  quantities $|S|$, $|D|$ and $\mmodel$ (the latter been fixed in the MAP
  context). 

\item Given the maximal number of clusters, partitions are equiprobable among the
  partitions with at most the specified maximal number of clusters, that is
  \begin{equation}
 p(\CS | k^{\max}_S) = \dfrac{1}{B(|S|,k_S^{\max})} , \quad p(\CD | k^{\max}_D) = \dfrac{1}{B(|D|,k_D^{\max})},
\end{equation}
where $B(|S|,k^{\max}_S)=\sum_{k=1}^{k^{\max}_S} S(|S|,k)$ is the sum of Stirling numbers of 
the second kind, i.e the number of ways of partitioning $|S|$ elements into $k$
non-empty subsets.  

At this step, the prior does not favor any particular structure in the
partition of vertices beside their number of clusters (partitions will low
number of clusters are favored over partitions with a high number of
clusters). It depends indeed only on $k^{\max}_S$ and $k^{\max}_D$ not on
the actual partitions. 

This is quite different from e.g. \cite{Kemp2006} where a Dirichlet
process is used as a prior on the number of clusters and on the distribution
of vertices on the clusters. Such a prior favors a structure with a few
populated clusters and several smaller clusters and penalizes balanced
clustering models. Our approach overcomes this issue owing to the choice of
its prior (see also below). 

\item  For a triclustering with $k_S$ source, $k_D$ destination clusters and $k_T$
  time segments, assignments of the $\mdata$ edges on the $k_S
  \times k_D \times k_T$ triclusters are equiprobable. 
It is known that the number of such assignments (i.e. the
$k_S\times k_D\times K_T$ numbers $\EC$ which sum to $\mdata$) 
is $\binom{\mdata+k_Sk_Dk_T-1}{k_Sk_Dk_T-1}$, leading to 
\begin{equation}
p(\EC|k_S,k_D,k_T) = \frac{1}{\dbinom{\mdata+k_Sk_Dk_T-1}{k_Sk_Dk_T-1}}.
\end{equation}
Notice that this prior penalizes a high number of triclusters. As the numbers
of vertex clusters are already penalized before (via the number of
partitions), this has mostly an effect on the number of time intervals $k_T$.
\item Similarly, for each source cluster $\cs_i$, the out-degrees of
  the vertices are chosen uniformly at random among the degree lists that
  sums to $\secmargin{i}$, as
  requested by the constraints (this holds
  also for destination clusters), which leads to 
  \begin{equation}
 p(\{\sdeg_{s}\}_{s\in\cs_i} | \EC , \CS) = \dfrac{1}{\dbinom{\secmargin{i} + |\cs_i| -1}{|\cs_i| -1}}, 
   \end{equation}
and similarly to
\begin{equation}
p(\{\ddeg_{d}\}_{d\in\cd_j} | \EC, \CD) = \dfrac{1}{\dbinom{\decmargin{j} + |\cd_j| -1}{|\cd_j| -1}}.
\end{equation}
  For a given assignment $\EC$, this prior penalizes large clusters
  (in terms of degree, i.e. high values of $|\cs_i|$ or $|\cd_j|$), or in other words, it favors balanced
  partitions (with clusters of the same sizes, again in terms of degree). For given partitions, the
  prior penalizes high marginal counts, in particular in large (degree) clusters. 
\end{enumerate}

Overall, the prior is rather flat, as it is uniform at each level of the hierarchy of the parameters.
It does not make strong assumptions and let the data speak for themselves, as
the prior terms vanish rapidly compared to the likelihood terms. Notice that
other prior distribution could be considered, especially if expert knowledge
is available. 

\subsection{The MODL criterion}\label{sec:modl-criterion}
The product of the prior distribution above and of likelihood term obtained in
the previous section results in a posterior
probability, the negative log of which is used to build the criterion
presented in Definition~\ref{def:MODL}.

\begin{definition}[MODL Criterion]\label{def:MODL}
  According to the MAP approach, the best adjustment of the model and the
  temporal data set $E$ is obtained when triclustering $\mathcal{M}$ is
  compatible with $E$ (according to Definition \ref{def:compatibility}) and
  minimizes the following criterion:
\begin{small}\begin{align}
&c(\mathcal{M}) = \log |S| + \log |D| + \log \mdata + \log B(|S|,k_S) + \log B(|D|,k_D) \nonumber \\
 &+\log \displaystyle \binom{\mdata+k_Sk_Dk_T-1}{k_Sk_Dk_T - 1}+ \displaystyle\sum_{i=1}^{k_S} \log \binom{\secmargin{i} + |\cs_{i}| -1}{|\cs_{i}| -1} + \displaystyle\sum_{j=1}^{k_D} \log \binom{\decmargin{j} + |\cd_j| -1}{|\cd_j| -1} \nonumber \\
 &+\log \mdata! - \displaystyle \sum_{i=1}^{k_S}\sum_{j=1}^{k_D} \sum_{l=1}^{k_T} \log \ec_{ijl}!  + \displaystyle \sum_{l=1}^{k_T} \log \tecmargin{l}!
\nonumber \\
 &  + \displaystyle \sum_{i=1}^{k_S} \log \secmargin{i}! - \displaystyle \sum_{s \in S} \log \sdeg_s! + \displaystyle \sum_{j=1}^{k_D} \log \decmargin{j}! - \displaystyle \sum_{d \in D} \log \ddeg_d!.
\label{eq:MODL}
\end{align}\end{small}
\end{definition}
It is important to note that the quality criterion is defined only for
parameters that are compatible with the data set $E$. This explains why only
$m$ appears directly in the criterion: the actual characteristics of the data
set influence indirectly the value of the criterion (for a given set of
parameters) via the compatibility equations from Definition
\ref{def:compatibility}. In particular, the degrees $\SDEG$ and $\DDEG$ are
fixed, and each triclustering $\CS$, $\CD$ and $\CT$ leads to a unique
compatible $\EC$. In this sense, the MODL criterion is really a triclustering
quality criterion. 

In addition, the evaluation criterion of Definition~\ref{eq:MODL} relies on
counting the number of possibilities for the model parameters and for the data
given the model. As negative log of probability amounts to a Shannon-Fano
coding length \citep{Shannon1948}, the criterion can be interpreted in terms
of description length.  The two first lines of the criterion correspond to the
description length of the triclustering $-\log P(\mathcal{M})$ (prior
probability) and the two last lines to the description length of the data
given the triclustering $-\log P(E|\mathcal{M})$ (likelihood).  Minimizing the
sum of these two terms therefore has a natural interpretation in terms of a
crude MDL (minimum description length) principle \citep{Grunwald2007}.
Triclustering fitting well the data get low negative log likelihood terms, but
too detailed triclusterings are penalized by the prior terms, mainly the
partition terms which grow with the size of the partitions and the assignment
parameters terms which grow with the number of triclusters.

\subsection{Optimization strategy}
The criterion $c(\mathcal{M})$ provides an exact analytic formula for the
posterior probability of the parameters $\mathcal{M}$, but the parameter space
to explore is extremely large. That is why the design of
sophisticated optimization algorithms is both necessary and meaningful. Such
algorithms are described by \cite{Boulle2010}.   

Interestingly while the assignment based representation allows one to define a
simple non informative prior on the parameters, it is not a realistic
representation for exploring the parameter space. Indeed there is no natural
and simple operator to move from one compatible assignment $\EC$ to another
one. On the contrary, working directly with the three partitions $\CS$, $\CD$
and $\CT$, and getting $\EC$ from the data (under the compatibility
constraints) is much more natural.

The criterion is indeed minimized using a greedy bottom-up merge heuristic. It
starts from the finest model, i.e the one with one cluster per vertex and one
interval per time stamp. Then merges of source clusters, of destination
clusters and of adjacent time intervals are evaluated and performed so that
the criterion decreases. This process is reiterated until there is no more
improvement, as detailed in Algorithm~\ref{alg:gbum}.

\begin{algorithm}
\caption{Greedy Bottom Up Merge Heuristic}
\label{alg:gbum}
\begin{algorithmic}
	\Require $\mathcal{M}$ (initial solution)
	\Ensure $\mathcal{M}^*\mbox{ ; } c(\mathcal{M}^*) \leq c(\mathcal{M})$
	\State $\mathcal{M}^* \gets \mathcal{M}$
	\While{$\mbox{solution is improved}$}
		\State $\mathcal{M}' \gets \mathcal{M}^*$
		\ForAll {merge $u$ between 2 source or destination clusters or adjacent time segments}
			\State $\mathcal{M}^+ \gets \mathcal{M}^*+u$
			\If {$c(\mathcal{M}^+)<c(\mathcal{M}')$}
				\State $\mathcal{M}' \gets \mathcal{M}^+$
			\EndIf
		\EndFor
		\If {$c(\mathcal{M}')<c(\mathcal{M}^*)$}
			\State $\mathcal{M}^* \gets \mathcal{M}'$ (improved solution)
		\EndIf
	\EndWhile
\end{algorithmic}
\end{algorithm}

The greedy heuristic may lead to computational issues and a naive
straightforward implementation would be barely usable because of a too high
algorithmic complexity. By exploiting both the sparseness of the temporal data
set and the additive nature of the criterion, one can reduce the memory
complexity to $O(m)$ and the time complexity to $O(m\sqrt{m}\log m)$.  The
optimized version of the greedy heuristic is time efficient, but it may fall
into a local optimum. This problem is tackled using the variable neighborhood
search (VNS) meta-heuristic \citep{Hansen2001}, which mainly benefits from
multiple runs of the algorithms with different random initial solutions to
better explore the space of models.  The optimized version of the greedy
heuristic as well as the meta-heuristics are described in details in
\cite{Boulle2010}.

\subsection{Simplifying the triclustering structure}\label{sec:simpl-tricl-struct}
When very large temporal data sets are studied, i.e. when $\mdata$ becomes
large compared to $|S|$ and $|D|$, the number of clusters of vertices and of
time stamps in the best triclustering may be too large for an easy
interpretation. This problem has been raised by \cite{White1976}, who suggest
an agglomerative method as an exploratory analysis tool in the context of
social networks analysis. We describe in this section a greedy
aggregating procedure that reduces this complexity in a principled way, using
only one user chosen parameter. 

The method we propose in this paper consists in merging successively the
clusters and the time segments in the least costly way until the triclustering
structure is simple enough for an easy interpretation. Starting from a
locally optimal set of parameters according to the criterion detailed in Equation~\eqref{eq:MODL}, 
clusters of source vertices, of destination vertices or time stamp ranks are
merged sequentially (in such way that time stamp partitions always respect the
order of the time stamps). At each step, the two clusters to merge are
the ones that induce the smallest increase of the value of the
criterion. This post-treatment is equivalent to an agglomerative
hierarchical clustering where the dissimilarity measure between two clusters
is the variation of the criterion due to this merge, as in the following
definition. 
\begin{definition}
\label{def:dissimilarity}
Let $\mathcal{M}$ be a triclustering and let $c_1$ and $c_2$ be two clusters
of $\mathcal{M}$ on the same variable (that is two source clusters, or two
destination clusters or two consecutive time stamp clusters). 

The MODL dissimilarity between $c_1$ and $c_2$ is given by
\begin{equation}
  \modldiss (c_1,c_2)=c(\mathcal{M}_{\text {merge } c_1 \text{ and }c_2})- c(\mathcal{M}),
\end{equation}
where $\mathcal{M}_{\text {merge } c_1 \text{ and }c_2}$ is the triclustering
obtained from $\mathcal{M}$ by merging $c_1$ and $c_2$ into a single cluster.

Appendix \ref{sec:interpr-diss-betw} provides some interpretations of this dissimilarity. 
\end{definition}

To handle the coarsening of a triclustering in practice, a measure of
informativeness of the triclustering is computed at each agglomerative
step of Algorithm~\ref{alg:gbum}. 
It corresponds to the percentage of informativity the triclustering has
kept after a merge, compared to a null model.
\begin{definition}[Informativity of a triclustering]\label{def:informativity}
  The null triclustering $\mathcal{M}_{\emptyset}$ has a single cluster of
  source vertices and a single cluster of destination vertices and one time
  segment. It corresponds to a stationary graph with no underlying
  structure. Given the best triclustering $\mathcal{M}^*$ obtained by optimizing the
  criterion defined in Definition~1, the informativity of a triclustering
  $\mathcal{M}$ is:

\begin{equation}
\tau(\mathcal{M})=\dfrac{c(\mathcal{M})-c(\mathcal{M}_{\emptyset})}{c(\mathcal{M}^*)-c(\mathcal{M}_{\emptyset})}.
\end{equation}
\end{definition}
By definition, $0 \leq \tau(\mathcal{M}) \leq 1$ for all triclusterings more
probable than the null triclustering. In addition, $\tau(\mathcal{M}_{\emptyset}) = 0$ and
$\tau(\mathcal{M}^*) = 1$.

The informativity is chosen (or monitored) by the analyst in order to stop the
merging process. This is the only user chosen parameter of our method. Notice
in particular that the merging process chooses automatically which variable to
coarsen: the user do not need to decide whether to reduce the number of
clusters on e.g. the source vertices versus the time stamps. 

In practice, the coarsening can be seen as a modification of
Algorithm~\ref{alg:gbum}. Rather than accepting a merge only if the quality criterion
is increased, the algorithm selects the best merge in term of the quality of the
obtained triclustering (in the inner for loop) and proceeds this way until the
triclustering is reduced to only one cluster or the informativity drops below
a user chosen value (in the outer while loop). 

\section{Experiments on artificial data sets}\label{sec:exper-artif-data}
Experiments have been conducted on artificial data in order to investigate the
properties of our approach. To that end, we generate artificial graphs with
known underlying time evolving structures (see
\cite{guigouresboulleetal2012triclustering-approach} for complementary
experiments on a graph with unbalanced clusters).

\subsection{Data sets}
Experiments are conducted on temporal graphs in which the edge
structure changes through time from a quasi-co-clique pattern where
edges are concentrated between different clusters to a quasi-clique
pattern where edges are concentrated inside clusters.

More precisely, we consider given a source vertex set $S$ and a target vertex
set $D$, both partitioned into $k$ balanced clusters, respectively $(A^S_i)_{1\leq
  i\leq k}$ and $(A^D_j)_{1\leq j\leq k}$. The time interval is
arbitrarily fixed to $[0,1]$. On this interval, a function $\Theta$ is defined
with values in the set of squared $k\times k$ matrices by:
\begin{equation}  
 \Theta(t) = \left\{
    \begin{array}{ll}
        \theta_{ii}(t) = \frac{0.9t + 0.1(1-t)}{k}, \\
        \theta_{ij}(t) = \frac{0.1t + 0.9(1-t)}{k(k-1)} \mbox{ when } i \neq j.
    \end{array}
\right.
\end{equation}   
The term $\theta_{ij}(t)$ can be seen as a connection probability between
a source vertex in cluster $A^S_i$ and a destination vertex in $A^D_j$ (this
is slightly more complex, as explained below). In particular, when $t=0$,
connections will seldom appear inside diagonal clusters, while they will
concentrate on the diagonal when $t=1$ (see Figure \ref{fig:patterns}).  

Given $k$ and $m$ a number of edges to generate, a temporal graph is obtained
by building each edge $e_l=(s_l, d_l, t_l)$ according to the following procedure:
\begin{enumerate}
\item $t_l$ is chosen uniformly at random in $[0,1]$;
\item the clusters indexes $(u_l,v_l)$ are chosen according to the categorical
  distribution on all the pairs $(i,j)_{1\leq  i\leq k, 1\leq   j\leq k}$
  specified by $\Theta(t_l)$ (that is $P(u_l=i, v_l=j)=\theta_{ij}(t_l)$);
\item $s_l$ is chosen uniformly at random in $A^S_{u_l}$ and $d_l$ is chosen
  uniformly at random in $A^D_{v_l}$. 
\end{enumerate}
\begin{figure}[ht!]
\centering
        \begin{subfigure}[b]{0.24\textwidth}
                \centering
                \includegraphics[width=\textwidth]{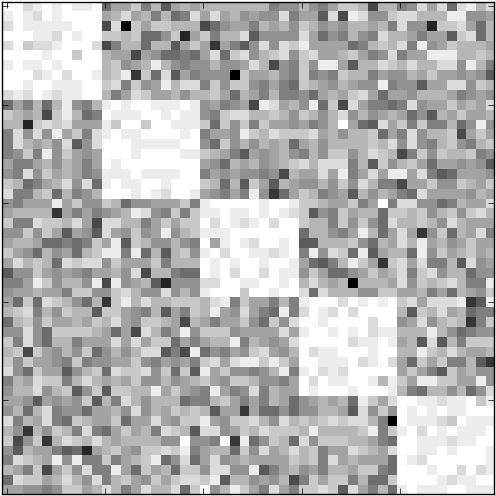}
                \caption{$[0,\epsilon]$}
                \label{fig:0}
        \end{subfigure}
				\,
        \begin{subfigure}[b]{0.24\textwidth}
                \centering
                \includegraphics[width=\textwidth]{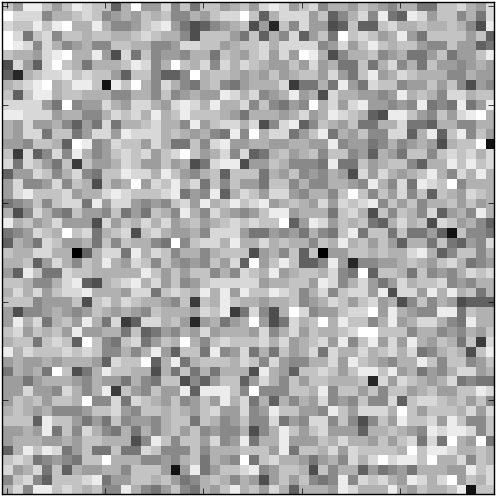}
                \caption{$[0.2,0.2+\epsilon]$}
                \label{fig:02}
        \end{subfigure}%
				\,
        \begin{subfigure}[b]{0.24\textwidth}
                \centering
                \includegraphics[width=\textwidth]{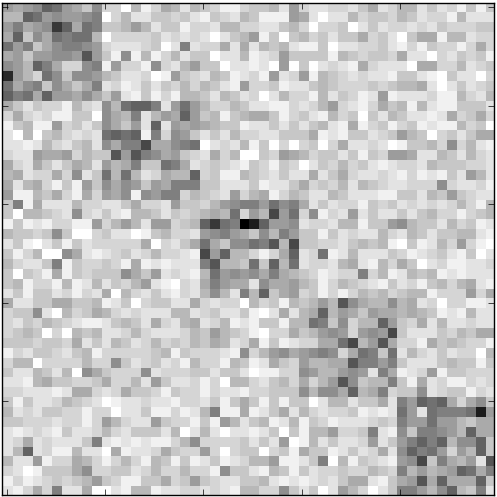}
                \caption{$[0.5,0.5+\epsilon]$}
                \label{fig:05}
        \end{subfigure}%
				\,
        \begin{subfigure}[b]{0.24\textwidth}
                \centering
                \includegraphics[width=\textwidth]{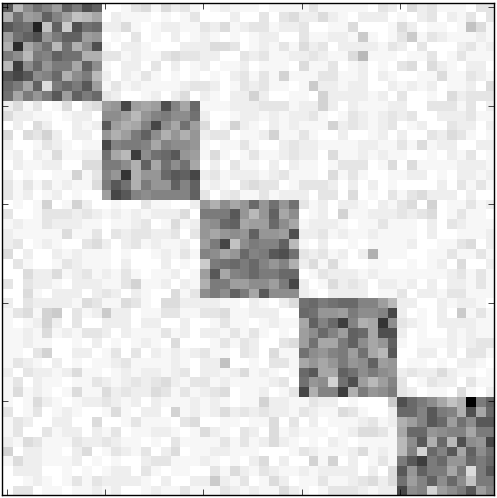}
                \caption{$[1-\epsilon,1]$}
                \label{fig:1}
        \end{subfigure}%
 \caption{Sample of graphs for four time values. $n=50$ vertices, $m=10^6$ edges, $k=5$ clusters, $\epsilon=10^{-2}$}%
\label{fig:patterns}
\end{figure}
Notice that this procedure is different from what is done in stochastic block
models \citep{Snijders2001} and related models as it aims at mimicking
repeated interactions. The procedure is also quite different from
the generative approach detailed in Section \ref{sec:gener-model-temp} and does
not favor our model.

Two additional methods are also used to make the data more complex. The first
one consists in randomly reallocating the three variables (source vertex,
destination vertex and time stamp) for a randomly selected subset of
edges. The reallocation is made uniformly at random independently on each
variable. The percentage of reallocated edges measures the difficulty of the
task. The second complexity increasing method (applied independently) consists
in shuffling the time stamps to remove the temporal structure from the
interaction graph. Finally, we use also Erd\H{o}s-R\'enyi random graphs with
time stamps chosen uniformly at random in $[0,1]$ to study the robustness of
the method.

\subsection{Results}
We report results with $k=5$ clusters, 50 source vertices and 50 destination
vertices. Edge number varies from $2$ to $2^{20}$ (considering all powers of
2). For a given number of edges, we generate 20 different graphs. On 10 of
them, we applied the reallocation procedure described above for 50 \% of the
edges. 

\begin{figure}[hbt!]%
\centering
        \begin{subfigure}[b]{\linewidth}
                \centering
                \includegraphics[width=0.75\linewidth,trim = 3cm 1cm 3cm 1cm, clip]{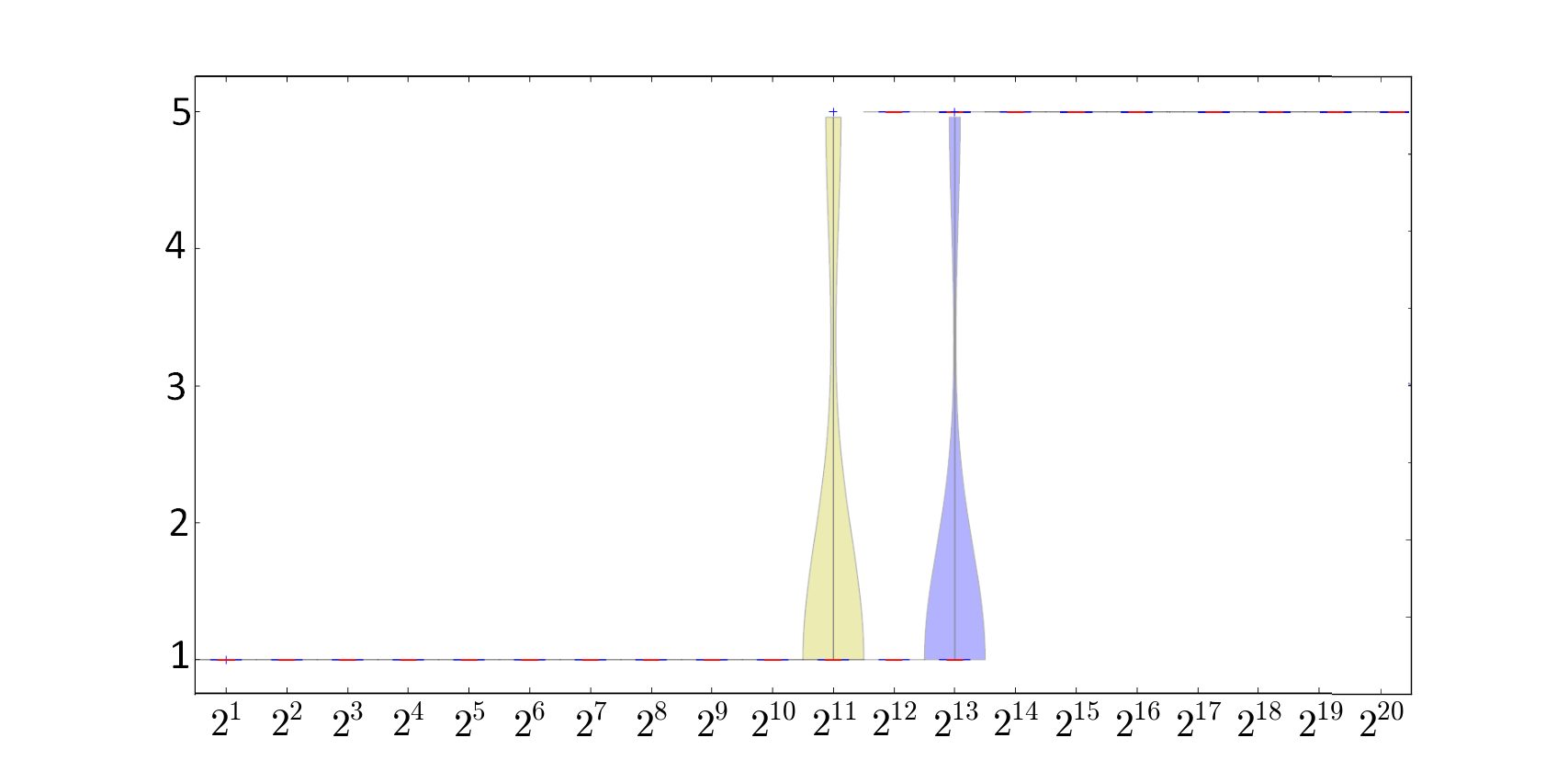}
                \caption{Violin plot of the number of clusters detected by the
                  proposed approach as a function of the number of
                  edges (yellow: no noise, violet: $50\%$ noise).}
                \label{fig:clusters}
        \end{subfigure}%
        \\
        \begin{subfigure}[b]{\linewidth}
                \centering
                \includegraphics[width=0.75\linewidth,trim = 3cm 1cm 3cm 1cm, clip]{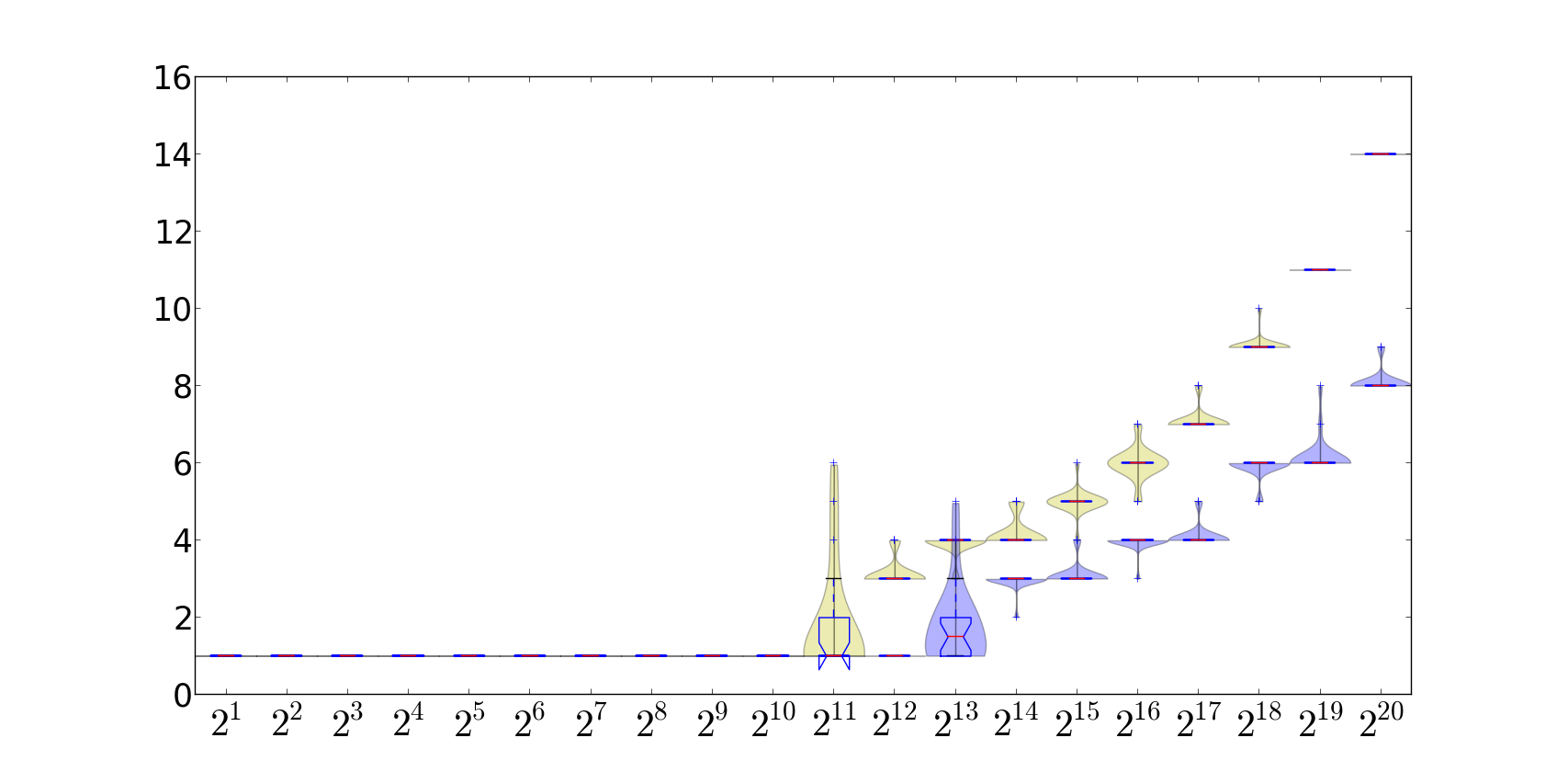}
                \caption{Violin plot of the number of time segments detected by the
                  proposed approach as a function of the number of
                  edges (yellow: no noise, violet: $50\%$ noise).}
                \label{fig:intervals}
        \end{subfigure}%
 \caption{Results for graphs with a temporal structure and two levels of noise
 (no noise and 50 \% of reallocated edges). Violin plots
 \citep{HintzeNelson1998ViolinPlot} combine a box plot and a density
 estimator, leading here to a better view of the variability of the results
 than e.g. standard deviation bars.}%
\label{fig:results}
\end{figure}

\paragraph{Temporal Graphs.} The Figures \ref{fig:clusters} and
\ref{fig:intervals} display respectively the average number of clusters of
vertices and the average number of time segments selected by the MODL
approach, in graphs in which the structure is preserved either completely (no
noise) or partly ($50 \%$ of reallocated edges). Bars show the standard
deviation of the number of clusters/segments. They are generally non visible
as the results are very stable excepted during the transition between the low
number of edges to the high number of edges. 

For a small number of edges (below $2^{10}$), the method does not discover any
structure in the data in the sense that the (locally) optimal triclustering has only one
cluster for each variable. The number of edges is too small for the method to
find reliable patterns: the gain in likelihood does not compensate the
reduction in a posteriori probability induced by the complexity of the
triclustering itself. Between $2^{11}$ and $2^{12}$, data are numerous enough
to detect clusters but too few to support the detection of the true underlying
structure (the results are somewhat unstable at this point and the actual
number of clusters discovered by the method varies between each generated
graph). Finally, beyond $2^{12}$ edges, we have enough edges to
retrieve the true structure. More precisely, the number of clusters of source
and destination vertices reaches the true number of clusters and their content
agree, while the number of time segments increases with the number of
edges. This shows the good asymptotic behavior of the method: it retrieves the
true actor patterns and exploits the growing number of data to better
approximate the smooth temporal evolution of the connectivity
structure. Indeed $\Theta(t)$ is a $C^\infty$ function with bounded (constant)
first derivatives and is therefore smooth, with no brutal changes.

Notice finally that the behavior of the method is qualitatively similar on the
noisy patterns as on the noiseless ones, but that the convergence to the
true structure and the growth of the number of temporal clusters are slower in
the noisy case, as expected. 

\paragraph{Stationary Graphs.} When the temporal structure is destroyed by the
time stamp shuffling, the method does not partition the time stamps, leaving
them in a unique cluster, regardless of the number of edges. Given enough

edges ($2^{13}$ without noise and $2^{15}$ with $50\%$ noise), vertex clusters are recovered perfectly. This shows the efficiency of
the regularization induced by the prior distribution on the parameters. As in
the case of the temporal graph, disturbing the structure via reallocating
edges postpone the detection of the clusters to a larger number of edges.

\paragraph{Random Graphs.} When applied to Erd\H{o}s-R\'enyi random graphs
with no structure (neither actor clustering, nor temporal evolution), the
method selects as the locally optimal triclustering the one with only one cluster on
each dimension, as expected for a non overfitting method.

\section{Experiments on a Real-Life data set}\label{sec:exper-real-life}

Experiments on a real-life data set have been conducted in order to illustrate
the usefulness of the method on a practical case.

\subsection{The London cycles data set}

The data set is a record of all the cycle hires in the Barclays cycle stations
of London between May 31st, 2011 and February 4th, 2012. The data are
available on the website of TFL\footnote{Transport for London,
  \url{http://www.tfl.gov.uk}}. The data set consists in $488$ stations and
$\mdata=4{.}8$ million journeys. It is modelled as a graph with the departure
stations as source vertices, the destination stations as destination vertices
and the journeys as edges, with time stamps corresponding to the hire time with
minute precision. In this data set $S=D$ (with $|S|=488$) as every station is
the departure station and the arrival station of some journeys. 

\subsection{Most refined triclustering}
By applying the proposed method\footnote{On a standard desktop PC, this takes
  approximately 50 minutes, with a maximal memory occupation of 4.5 GB.} to
this data set, we obtain $296$ clusters of source stations, $281$ clusters of
destination stations and 5 time stamp clusters. Most of the clusters consist
in a unique station, leading to a very fine-grained clustering on the
geographical/spatial point of view. This is not the result of some form of
overfitting: due to the very large number of bicycle hires compared to the
number of stations, the distributions of edges coming from/to the vertices are
characteristic enough to distinguish the stations, in particular
because many journeys are locally distributed around a source station. On the
contrary, and perhaps surprisingly, the temporal dynamic is quite simple as
only $5$ time stamp clusters are identified. We label them as follows: the
\textit{morning} (from 7.06AM to 9.27AM), the \textit{day} (from 9.28AM to
3.25PM), the \textit{evening} (from 3.26PM to 6.16PM), the \textit{night}
(from 6.17PM to 4.12AM) and the \textit{dawn} (from 4.13AM to 7.05AM).

\subsection{Simplified triclustering}
We apply the exploratory post-processing described in Section
\ref{sec:simpl-tricl-struct} in order to study a simplified
triclustering. Clusters of stations are successively merged until obtaining
$20$ clusters of both departure and destination stations while the number of
time stamp clusters remains unchanged. By applying this post-processing
technique, $70\%$ of the informativity of the most refined triclustering is retained
(see Definition \ref{def:informativity}). Notice that the merging algorithm is
not constrained to avoid merging time intervals and/or to balance departure
and destination clusters. On the contrary, each merging
step is chosen optimally between all the possible merges on each of the three
variables available at this stage. This shows that while the temporal
structure is simple, it is very significant on a statistical point of view.

While the data set does not contain explicit geographic information, a
detailed analysis of the clusters reveals that the clustered stations are in
general geographically correlated. This is a natural phenomenon in a bike
share system where short journeys are favored both by the pricing structure
and because of the physical effort needed to travel from one point to
another. A notable exception is observed for the cycle stations in front of
Waterloo and King's Cross train stations (white discs on
Figure~\ref{fig:source:clusters}) that have been grouped together while they are quite
distant. This specific pattern is detailed and interpreted in
Section~\ref{sec:detailed-visualization}, using an appropriate
visualization method.

\begin{figure}[htbp]
  \centering
  \includegraphics[width=0.9\textwidth]{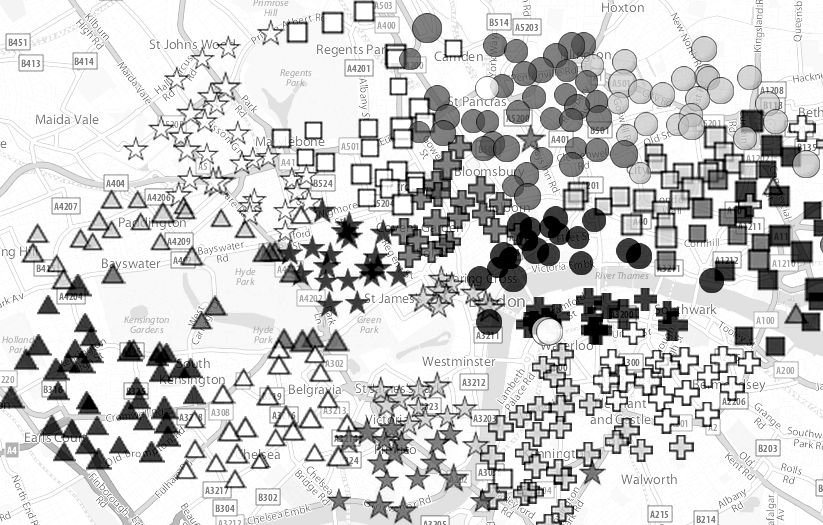}
 \caption{Clusters of source stations: each station is represented by a symbol
   whose shape and level of gray is specific to the corresponding source cluster.}
  \label{fig:source:clusters}
\end{figure}

The triclusterings obtained by our method are not constrained to yield
identical results on $S$ and $D$ even if $S=D$ (which is the case here). This
would be an important limitation as it would constraint an actor to have the
same role as a source than as a destination. In the bike share data set, we
obtain comparable but not identical clustering structures on the set of source and
destination vertices. The main notable difference lies on the segmentation of
the financial district of London: one single destination cluster covers the
area while it is split into two source clusters (the two types of gray squares on 
Figure~\ref{fig:source:clusters} form the source clusters, while most red
discs on the right hand side of Figure \ref{fig:mi:notime} form the
destination cluster). 

\subsection{Detailed Visualization}\label{sec:detailed-visualization}
The triclustering obtained with our method can help understanding the
corresponding temporal data set, in particular when it is used to build
specialized visual representations, as illustrate below. 

In order to better understand the partition of the stations, we investigate
the distribution of journeys originating from (resp. terminating to) the
clusters. To that end, we study the contribution to the mutual information of
each pair of source/destination stations. We first define more formally the
distributions under study. We denote $\scprob$ the probability
distribution on $\{1,\ldots,k_S\}$ given by
\begin{equation}
\scprob(\{i\})=\frac{\secmargin{i}}{\mdata}.
\end{equation}
It corresponds to the empirical distribution of the clusters in the data
set. Similarly, we denote $\dcprob$ the probability distribution on
$\{1,\ldots,k_D\}$ given by
\begin{equation}
\dcprob(\{j\})=\frac{\decmargin{j}}{\mdata}.
\end{equation}
Finally, the joint distribution $\sdcprob$ on $\{1,\ldots,k_S\}\times
\{1,\ldots,k_d\}$ is given by
\begin{equation}
\sdcprob(\{(i, j)\})=\frac{\sum_{l=1}^{k_T}\ec_{ijl}}{\mdata}.
\end{equation}
To measure the dependencies between the source and destination vertices at the
cluster level, we use the \emph{mutual information} \citep{cover} between the cluster
distribution, that is
\begin{equation}\label{eq:MutualInfo}
MI_C^{S,D}=\sum_{i=1}^{k_S}\sum_{j=1}^{k_D}\sdcprob(\{(i, j)\})\log
\frac{\sdcprob(\{(i, j\})}{\scprob(\{i\})\dcprob(\{j\})}.
\end{equation}
Mutual information is necessarily positive and its normalized version (NMI) is
commonly used as a quality measure in the co-clustering problems
\citep{Strehl2003}. Here, we only focus on the contribution to mutual
information of a pair of source/destination clusters. This value can be
either positive or negative according to whether the observed joint
probability of journeys $\sdcprob(\{(i, j\})$ is above or below the expected
probability $\scprob(\{i\})\dcprob(\{j\})$ in case of independence. Such a measure
quantifies whether there is a lack or an excess of journeys between two
clusters of stations in comparison with the expected number.

\begin{figure}[htbp]%
\centering
  \includegraphics[width=0.9\textwidth]{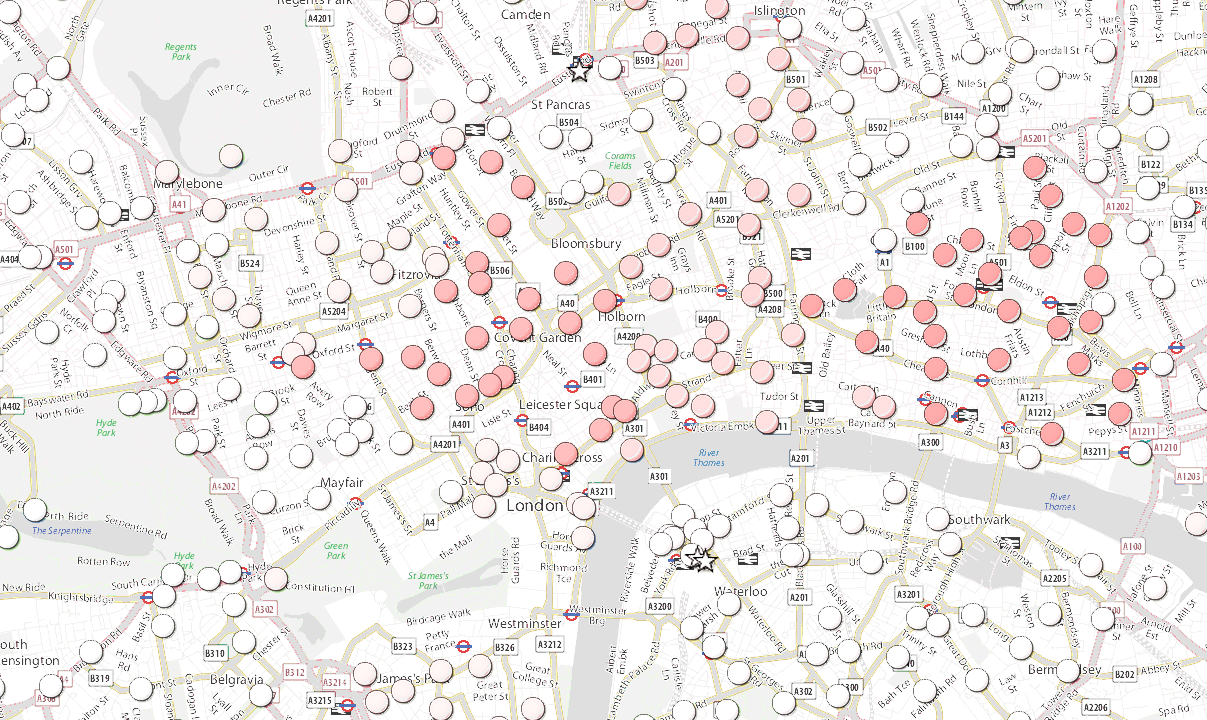}
 \caption{Destination cluster contributions to the mutual information between
   the source cluster 'Waterloo/King's Cross' (stations drawn using stars) and all the destination
   clusters. Within a destination cluster, all stations share the same
   color whose intensity is proportional to the contribution of the cluster
   to the mutual information. Positive contributions are represented in red,
   negative in blue. The present figure shows mainly positive or null contributions (no blue circles).}
\label{fig:mi:notime}
\end{figure}

For instance, Figure~\ref{fig:mi:notime} shows an excess of journeys from the
\textit{Waterloo} and \textit{King's Cross} train stations to the central
areas of London. Both train stations being major intercity railroad stations,
we can assume that people there have the same behavior and all converge to the
same points in London: the business districts. This convergence pattern
explains why distant cycle stations can be grouped in the same cluster.

In this first analysis, the time variable is not taken into account. It can be
integrated into a visualization by considering for instance the dependency
between the time stamp clusters on one hand and pairs of source and
destination clusters on the other hand. We first define $\tcprob$ the
probability distribution on $\{1,\ldots,k_T\}$ by
\begin{equation}
\tcprob(\{l\})=\frac{\tecmargin{l}}{\mdata}.
\end{equation}
The full joint distribution on the clusters is given by the probability
distribution $\fullcprob$ on $\{1,\ldots,k_S\}\times
\{1,\ldots,k_D\}\times\{1,\ldots,k_T\}$ given by
\begin{equation}
\fullcprob(\{(i, j, l)\})=\frac{\ec_{ijl}}{\mdata}.
\end{equation}
Then we display the individual contributions to the mutual information between
pairs of source/destination clusters and time clusters:
\begin{equation}
MI_C^{(S,D),T}=
\sum_{i=1}^{k_S}\sum_{j=1}^{k_D}\sum_{l=1}^{k_T}\fullcprob(\{(i, j, l)\})\log
\frac{\fullcprob(\{(i, j, l\})}{\sdcprob(\{(i, j)\})\tcprob(\{l\})}.
\label{eq:MutualInfoCoupleTime}
\end{equation}

Similarly to the previous measure, this one aims at showing the pairs of
clusters between which there is an excess of traffic compared to the usual
daily traffic between these stations and the usual traffic at this period in
London. For example, for the source cluster \textit{Waterloo/King's Cross},
the traffic is higher than expected on mornings to the destination clusters
located in the center of London (see Figure~\ref{fig:HPday:morning}). By contrast
there is a lack of evening journeys (see Figure~\ref{fig:HPday:evening}). These
results are not really surprising because we can assume that in the mornings,
people use the cycles as a mean of transport to their office rather than as a
leisure activity.

\begin{figure}[htbp]%
\centering
  \includegraphics[width=.9\textwidth]{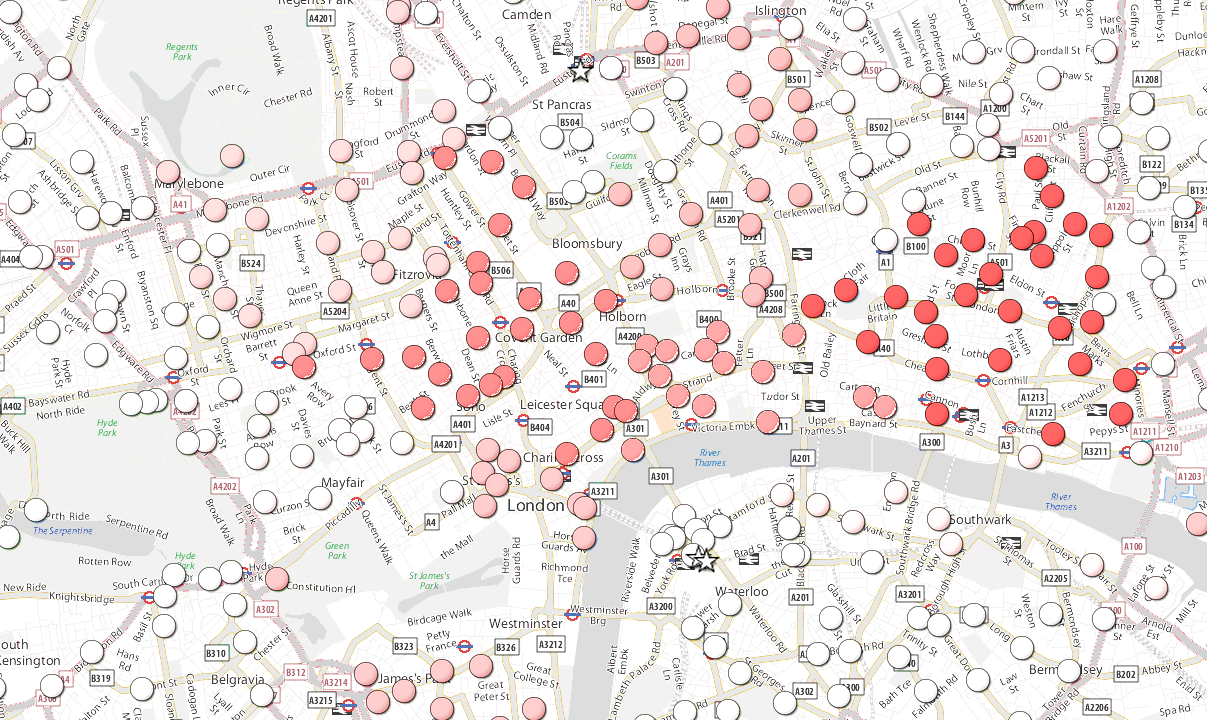} 
  \caption{Each station is colored according to the contribution of its
    destination cluster and of the source cluster
    \textit{Waterloo/King's Cross} (stations drawn using stars) to the mutual
    information between the source/destination pairs and the time segments. As in
    Figure \ref{fig:mi:notime}, color intensity measures the absolute value of
    the contribution, while the sign is encoded by the hue (red for positive
    and blue for negative). In the present figure, the time segment is the
    morning one, with mainly positive or null contributions (no blue circles).}
\label{fig:HPday:morning}
\end{figure}

\begin{figure}[htbp]%
\centering
  \includegraphics[width=.9\textwidth]{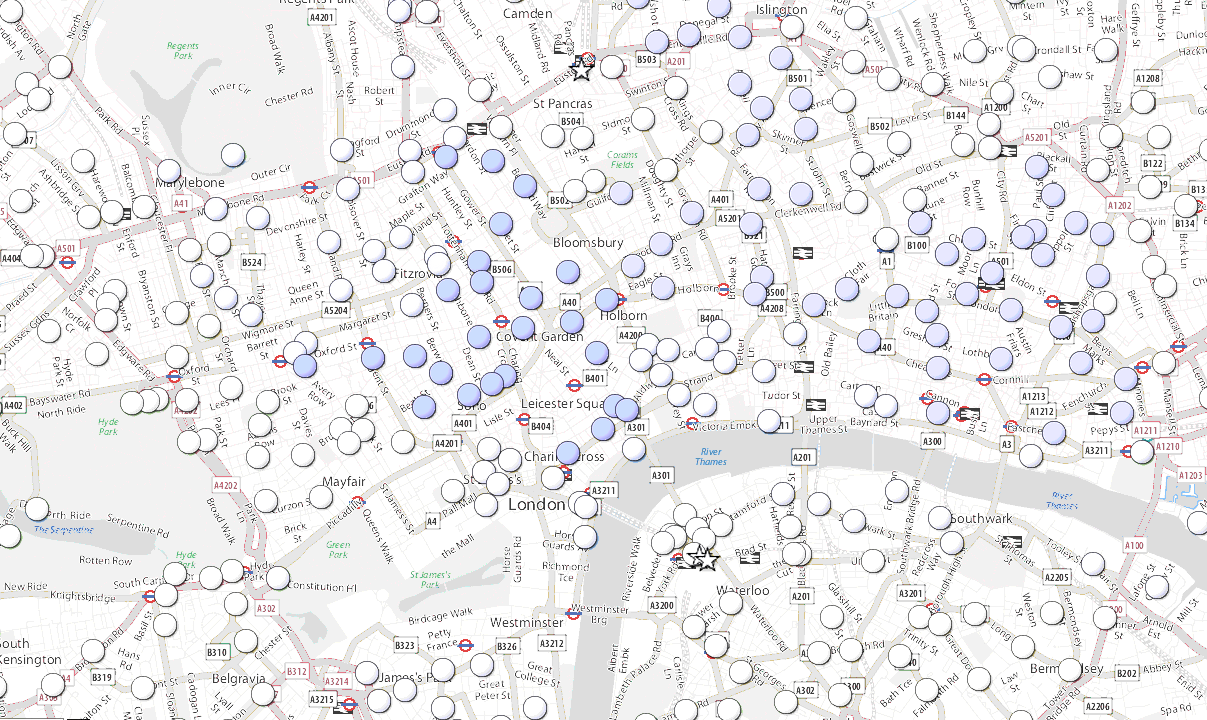}
  \caption{Mutual information contribution for the evening time segment. See
    Figure \ref{fig:HPday:morning} for details.
    The present figure shows mainly negative or null contributions (no red circles).}
\label{fig:HPday:evening}
\end{figure}

\section{Conclusion} \label{sec:conclusion}
This paper introduces a new approach for discovering patterns in time evolving
graphs, a type of data in which interactions between actors are time
stamped. The proposed approach, based on the MODL methodology, operates by
grouping in clusters source vertices, destination vertices and time stamps in
the same procedure. Time stamps clusters are constrained to respect their
ordering, leading to the construction of time intervals. The proposed method
is related to co-clustering in that we consider the graph as a set of edges
described by three variables: source vertices, destination vertices and
time. All of them are simultaneously partitioned in order to build time interval
on which the interactions between actors can be summarized at the cluster
level. This approach is particularly interesting because it does not require
any data preprocessing, such as an aggregation of time stamps or a selection
of significant edges. Moreover the evolving structure of the graph is tracked
in one unique step, making the approach more reliable to study the temporal
graphs. Its good properties have been assessed with experiments on artificial
data sets. The method is reliable because it is resilient to noise and
asymptotically finds the true underlying distribution. It is also suitable in
practical cases as illustrated by the study on the cycles renting system of
London. In future works, such a method could be extended to co-clustering in
k-dimensions, adding labels to the vertices or another temporal feature, such
as the day of week or the duration of an interaction for example. This would
allow us for instance to model the cycles renting system in more details by
taking into account both the departure time and the arrival time of a bike ride. 
A more ambitious goal would be to allow more complex
clustering structures. Indeed in this paper, vertex clusters are time
independent, while it would make sense to allow some time dependencies to the
clustering. In our framework, a possibility would be to retain more clusters
during a some time intervals and less during others, when the structure is
simplified. In other words, two clusters of vertices could be merged on
interval $[t_1,t_2]$ but kept separated during interval $[t_2,t_3]$. This
would allow tracking the complexity of interaction patterns in a non uniform
way through time, rather in the implicitly uniform way we handle them in the
current method.

\section*{Acknowledgment}
The authors thank the anonymous reviewers and the associate editor for their
valuable comments that helped improving this paper. 

\appendix

\section{Interpretations of the dissimilarity between two clusters}\label{sec:interpr-diss-betw}
Interestingly, the dissimilarity given in Definition~\ref{def:dissimilarity}
receives several interpretations. It corresponds to a loss of coding length
(when the MODL criterion is interpreted as a description length), a loss of
posterior probability of the triclustering given the data (see
Proposition~\ref{th:dissimilarity}), and asymptotically to a divergence
between probability distributions associated to the clusters (see
Proposition~\ref{th:divergence}).

\begin{proposition}
\label{th:dissimilarity}
The exponential of the dissimilarity between two clusters, $c_1$ and $c_2$,
gives the inverse ratio between the probability of the simplified triclustering
given the data set and the probability of the original
triclustering given the data set:
\begin{equation}
P(\mathcal{M}|E)=e^{\modldiss(c_1,c_2)}  P(\mathcal{M}_{\text {merge } c_1 \text{ and }c_2}|E).
\end{equation}
\end{proposition}

Asymptotically - i.e when the number of edges tends to infinity - the
dissimilarity between two clusters is proportional to a generalized
Jensen-Shannon divergence between two distributions that characterize the
clusters in the triclustering structure. To simplify the discussion, we give
only the definition and result for the case of source clusters, but this can
be generalized to the two other cases.

\begin{definition}
Let $\mathcal{M}$ be a triclustering. For all $i\in\{1,\ldots,k_S\}$ we denote
\begin{equation}
\sprob_i=\left(\frac{\ec_{ijl}}{\secmargin{i}}\right)_{1\leq j\leq k_D, 1\leq
    l\leq k_T}.
\end{equation}
The matrix $\sprob_i$ can be interpreted as a probability distribution over
$\{1, \ldots, k_D\}\times\{1, \ldots, k_T\}$. It characterizes $\cs_i$ as a
cluster of source vertices as seen from clusters of destination vertices and
of time stamps. 

We denote $\sprob$ the associated marginal probability distribution obtained by
\begin{equation}
\sprob=\left(\frac{\sum_{i=1}^{k_S}\ec_{ijl}}{\sum_{i=1}^{k_S}\secmargin{i}}\right)_{1\leq j\leq k_D, 1\leq
    l\leq k_T}.
\end{equation}
Obviously, we have
\begin{equation}
\sprob=\sum_{i=1}^{k_S}\pi_i\sprob_i,
\end{equation}
where
\begin{equation}
\pi_i=\frac{\secmargin{i}}{\sum_{k=1}^{k_S}\secmargin{k}}.
\end{equation}
\end{definition}

\begin{proposition}
  \label{th:divergence}
Let $\mathcal{M}$ be a triclustering and let $\cs_i$ and $\cs_k$ be two source
clusters. Then
\begin{equation}
  \dfrac{\modldiss(\cs_i,\cs_k)}{\mmodel}\underset{\mmodel\to +\infty}{\longrightarrow} (\pi_i+\pi_k) JS^{\alpha_i,\alpha_k} (\sprob_i,\sprob_k), 
  \label{eq:KL}
\end{equation}
with
\begin{equation}
 JS^{\alpha_i,\alpha_k} (\sprob_i,\sprob_k)=\alpha_i KL(\sprob_i || \alpha_i \sprob_i + \alpha_k \sprob_k) + \alpha_k KL(\sprob_k || \alpha_i \sprob_i + \alpha_k \sprob_k),
\end{equation}
and where $\alpha_i$ and $\alpha_k$ are the normalized mixture coefficients
such as $\alpha_i = \frac{\pi_i}{\pi_i+\pi_k}$ and $\alpha_k =
\frac{\pi_k}{\pi_i+\pi_k}$
      \end{proposition}
      \begin{proof}
$JS$ is the generalized Jensen-Shannon Divergence \citep{Lin1991} and $KL$, the Kullback-Leibler Divergence. The full proof is left out for brevity and relies on the Stirling approximation: $\log n!= n \log(n) - n + O(\log n)$, when the difference between the criterion value after and before the merge is computed.
\end{proof}

The Jensen-Shannon divergence has some interesting properties: it is a
symmetric and non-negative divergence measure between two probability
distributions. In addition, the Jensen-Shannon divergence of two identical
distributions is equal to zero. While this divergence is not a metric, as it
is not sub-additive, it has nevertheless the minimal properties needed to be
used as a dissimilarity measure within an agglomerative process in the context
of co-clustering \citep{Slonim1999}.

\bibliographystyle{plainnat}
\bibliography{biblio}

\end{document}